\documentclass[sigconf]{acmart}
\AtBeginDocument{%
  }

\setcopyright{acmlicensed}
\copyrightyear{2018}
\acmYear{2018}
\acmDOI{XXXXXXX.XXXXXXX}
\acmConference[Conference acronym 'XX]{Make sure to enter the correct
  conference title from your rights confirmation email}{June 03--05,
  2018}{Woodstock, NY}
\acmISBN{978-1-4503-XXXX-X/2018/06}

\usepackage{enumitem}
\usepackage[ruled,linesnumbered]{algorithm2e}
\newtheorem*{theorem*}{Theorem}
\usepackage{listings}
\usepackage{multirow}

\usepackage{tikz}
\usepackage{amsmath}
\usetikzlibrary{arrows.meta, shapes.geometric, positioning, fit, backgrounds, calc}
\definecolor{clrblue}{RGB}{74,144,226}
\definecolor{clrorange}{RGB}{245,166,35}
\definecolor{clrgreen}{RGB}{184,233,134}
\definecolor{clrgray}{RGB}{155,155,155}

\begin{document}

\title{Serving the Long Tail: Training-Free LLM Candidate Generation
  for Vacation Rental Marketplaces}
 
\author{Syed Mohammed Arshad Zaidi}
\email{syzaidi@expediagroup.com}
\affiliation{%
  \institution{Expedia Group}
  \city{Austin}
  \state{Texas}
  \country{USA}
}
 
\author{Eric Rincon}
\email{erincon@expediagroup.com}
\affiliation{%
  \institution{Expedia Group}
  \city{Austin}
  \state{Texas}
  \country{USA}
}
 
\author{Shayan Hassantabar}
\email{shassantabar@expediagroup.com}
\affiliation{%
  \institution{Expedia Group}
  \city{Austin}
  \state{Texas}
  \country{USA}
}
 
\renewcommand{\shortauthors}{Zaidi et al.}

\begin{abstract}
Vacation rental marketplaces face a structural imbalance on the supply
side: a small fraction of properties receive most user interactions,
while the long tail of new, niche, and seasonal listings generates
too little behavioral signal for collaborative filtering to serve
effectively. At Vrbo, item-based k-nearest neighbors (IBKNN) is a core candidate generation channel, but leaves tens of thousands of properties with no candidates and produces weak neighborhoods for sparsely interacted ones. We present a training-free, LLM-based
candidate generation pipeline that complements IBKNN using static
property metadata alone. An off-the-shelf LLM synthesizes diverse
semantic queries per property, a pre-trained text encoder embeds
them, and an approximate nearest-neighbor index retrieves candidates
from an 11.7M-property catalog. A \emph{Union fusion} strategy
merges these with IBKNN while preserving the behavioral channel's
ordering, guaranteeing no degradation on well-served properties, and
a downstream learning-to-rank model re-scores the fused pool.
Evaluated on 1.6M focal properties, the system extends candidate
coverage to tens of thousands of properties IBKNN cannot reach,
delivers its largest gains on the long-tail segment where behavioral
methods are weakest, and matches or beats IBKNN at every $K$ on
shared properties. A downstream learning-to-rank stage further lifts
the fused pool, yielding a complete candidate generation and
re-ranking stack that serves the long tail without regressing
well-served properties. We additionally show that Union fusion collapses the recall gap between a 3B open-weights LLM and frontier API-based models from 27--46\% to under 1\%, supporting self-hosted small-model deployment at marketplace catalog scale.
\end{abstract}

 
\maketitle

\section{Introduction}
Vacation rental marketplaces like Vrbo connect travelers searching
for accommodations with hosts who list their properties on the
platform. Every interaction implicates both sides: travelers expect
recommendations that help them discover relevant listings beyond
what direct search surfaces, and hosts expect their properties to
receive fair exposure regardless of listing age, seasonality, or
niche appeal. Reconciling these two expectations is the defining
challenge of marketplace recommendation, and it is resolved -- or
not -- at the candidate generation (CG) layer, the upstream
component that decides which subset of the catalog is even eligible
to be ranked and shown to any given user.
 
On Vrbo, item-to-item recommendations power surfaces such as the
``similar properties'' carousel shown alongside a focal listing.
The existing CG stack for this task combines multiple
complementary channels, with item-based k-nearest neighbors (IBKNN)
as a core channel. IBKNN scores pairs of properties by their
behavioral co-occurrence in user clickstreams, boosted by amenity
overlap. On properties with rich interaction history, it produces
strong, tightly geo-localized neighborhoods at low computational
cost. However, IBKNN inherits the fundamental limitation of collaborative-filtering: new listings and properties in emerging markets have no co-occurrence data and cannot be served at all.\  Long-tail properties with only a handful of interactions yield noisy, unreliable
neighborhoods; in our catalog, roughly 22\% of properties receive
five or fewer interactions per month, yet collectively account for
less than 1\% of total interactions.\ Finally, properties that are
semantically similar but lack direct co-view history remain
invisible to co-occurrence, even when they would be excellent
substitutes from a user's perspective. These are not merely recall
gaps for the CG layer; they translate directly into an
\emph{exposure imbalance} on the supply side, where the same
popular properties keep being recommended while the long tail is
starved of the visibility needed to bootstrap into popularity.
 
To address the aforementioned problem, our approach is to complement IBKNN with a second channel that does not depend on behavioral signal at all.\ We treat a large language model (LLM) as a \emph{query generator}: given only a property's static metadata, the LLM synthesizes a small set of diverse natural-language search queries that describe what the property offers from different perspectives.\ These queries are encoded by a pre-trained dense text encoder and dispatched against an approximate nearest-neighbor index of the full property catalog; results are aggregated and geographically re-ranked. The resulting candidates are merged with IBKNN through a \emph{union fusion} strategy that preserves the behavioral channel's ranking by construction and fills remaining slots with semantically-retrieved alternatives.\ A downstream learning-to-rank model then re-scores the fused pool.
 
\paragraph{Contributions.} This paper makes the following
contributions, aimed at practitioners building CG systems for
two-sided marketplaces:
\begin{itemize}[leftmargin=*]
  \item We describe a \textbf{training-free, metadata-driven LLM
  candidate generation pipeline}, paired with a downstream
  learning-to-rank stage, evaluated at scale (1.6M focal properties
  over an 11.7M-property catalog). The pipeline is stateless and
  applicable to any property with parseable metadata -- including
  brand-new listings with zero interaction history -- and requires
  no labeled data, no fine-tuning, and no user histories.
  \item We introduce \textbf{union fusion}, a channel-combination
  strategy with a formal recall-preservation guarantee: union recall
  $\geq$ IBKNN recall for every focal property (the listing for which
  candidates are being generated) at every $K$. This
  makes union strictly preferable to Reciprocal Rank Fusion when
  shipping a new channel alongside a mature behavioral baseline,
  and we show empirically why.
  \item We report a \textbf{per-segment analysis} showing that the
  LLM channel's lift concentrates exactly where IBKNN is weakest,
  and extends coverage to tens of thousands of cold-start properties
  that the behavioral channel cannot serve.
  \item We show that \textbf{union fusion makes the choice of
  query-generator LLM robust}: the 27--46\% standalone recall gap
  between a 3B open-weights model and frontier API-based models
  collapses to under 1\% once candidates are fused with IBKNN under
  Union. This makes self-hosted small-model deployment a defensible
  full-scale choice, avoiding the operational ceiling imposed by
  frontier API rate limits.
\end{itemize}

\section{Related Work}\label{sec:related}
Item-based collaborative filtering has been a staple of large-scale
recommender systems since its introduction by Sarwar
et~al.~\cite{sarwar2001item} and its early deployment at scale by
Linden et~al.~\cite{linden2003amazon} at Amazon. The approach
remains competitive for candidate generation in modern stacks
because it is cheap to compute, interpretable, and produces tight
neighborhoods on well-interacted items. A long line of work has
extended co-occurrence-based CG with neural representations:
Item2Vec~\cite{barkan2016item2vec} adapted the skip-gram framework
to learn dense item embeddings from co-occurrence data; the
sampling-bias-corrected two-tower retrieval model of Yi
et~al.~\cite{yi2019sampling} established the dual-encoder paradigm
for large-corpus item recommendation; and PinSage
\cite{ying2018pinsage} demonstrated graph convolutional networks at
web scale for item-to-item retrieval at Pinterest. Despite their
methodological diversity, these approaches share a common dependency
on user-item interaction signal, and all degrade -- to varying
degrees -- on long-tail and cold-start items.
 
Large language models have more recently been applied to
recommendation along several axes. TALLRec~\cite{bao2023tallrec}
proposed an efficient LoRA-based framework to align pre-trained LLMs
with the recommendation task, demonstrating that lightweight tuning
can produce competitive sequential recommenders. A parallel line of
work pursued generative retrieval: TIGER~\cite{rajput2023tiger}
replaced the standard embed-then-nearest-neighbor pipeline with an
autoregressive transformer that decodes semantic identifiers
directly. Both directions, however, require model training --
either parameter-efficient fine-tuning or full encoder-decoder
pre-training -- and do not natively integrate with existing
dense-retrieval infrastructure. A third, more recent direction
sidesteps both costs by using LLMs purely as \emph{query generators}.
Kim et~al.~\cite{kim2025querec} proposed QueRec, which prompts an
LLM to generate personalized search queries from a user's
interaction history and retrieves items through dense search over
the query embeddings, with no LLM fine-tuning required.
 
Cold-start mitigation through content and side information has been
studied extensively in the lodging domain in particular.
Hotel2vec~\cite{sadeghian2019hotel2vec} is the closest prior work in
spirit to ours: it fuses click-based hotel embeddings with
structured attribute encoders (location, amenities, star rating) to
improve recommendation quality, with explicit gains on cold-start
hotels that lack click history. The broader pattern is consistent:
content-based and hybrid approaches help where behavioral signal is
sparse, but typically require training a dedicated encoder on
paired interaction-attribute data.
 
Our work differs from each thread in a specific way: unlike
behavioral CG channels, we require no interaction data; unlike
TALLRec~\cite{bao2023tallrec} and TIGER~\cite{rajput2023tiger},
we have no trainable components in the CG stage; we share the
query-then-retrieve paradigm with QueRec~\cite{kim2025querec} but
target \emph{item-to-item} CG from static metadata rather than
user-history-conditioned sequential recommendation; and unlike
Hotel2vec~\cite{sadeghian2019hotel2vec}, we let a general-purpose
LLM verbalize metadata into queries rather than training a
dedicated encoder.

\section{Methodology}\label{sec:method}
\begin{figure*}[t]
\centering
\resizebox{\linewidth}{!}{%
\begin{tikzpicture}[
  font=\footnotesize,
  >=Stealth,
  chip/.style={
    rectangle, rounded corners=3pt,
    text centered,
    fill=clrblue!15, draw=clrblue!70, thin,
    font=\scriptsize, inner sep=2pt
  },
  stage/.style={
    rectangle, rounded corners=8pt,
    minimum width=1.8cm, minimum height=1.1cm,
    text centered, align=center,
    fill=clrblue!40, draw=clrblue!85, thick,
    font=\footnotesize, inner sep=3pt
  },
  stageemph/.style={
    rectangle, rounded corners=8pt,
    minimum width=1.8cm, minimum height=1.1cm,
    text centered, align=center,
    fill=clrblue!65, draw=clrblue!95, thick,
    font=\footnotesize, inner sep=3pt
  },
  fusion/.style={
    rectangle, rounded corners=8pt,
    minimum width=1.8cm, minimum height=1.1cm,
    text centered, align=center,
    fill=clrorange!40, draw=clrorange!85, thick,
    font=\footnotesize, inner sep=3pt
  },
  outbox/.style={
    rectangle, rounded corners=8pt,
    minimum width=1.8cm, minimum height=1.1cm,
    text centered, align=center,
    fill=clrgreen!60, draw=clrgreen!80!black, thick,
    font=\footnotesize, inner sep=3pt
  },
  arr/.style={->, line width=0.75pt, color=black},
]

\node[chip, minimum width=1.65cm] (title) at (0, 2.55) {Title};
\node[chip, minimum width=1.65cm] (desc)  at (0, 2.05) {Description};
\node[chip, minimum width=1.65cm] (loc)   at (0, 1.55) {City, State};
\node[chip, minimum width=1.65cm] (cap)   at (0, 1.05) {Capacity};
\node[chip, minimum width=1.65cm] (amen)  at (0, 0.55) {Amenities};

\begin{scope}[on background layer]
  \node[draw=clrblue!60, fill=clrblue!8, rounded corners=6pt,
        fit=(title)(amen), inner sep=6pt,
        label={[font=\scriptsize, text=black, yshift=-2pt]above:Property metadata}
       ] (mdbox) {};
\end{scope}

\node[chip, minimum width=1.75cm, text width=1.7cm, align=center, minimum height=0.9cm]
  (clicks) at (0, -1.8) {User clickstream\\ (co-occurrences)};

\node[stageemph] (llm) at (2.9, 1.55) {LLM\\ query gen.};

\node[chip, minimum width=2.5cm, text width=2.35cm, align=center, minimum height=0.55cm,
      draw=none, fill=none, text=white]
  (q1ph) at (5.8, 2.35) {``Austin lakefront\\ cabin with hot tub''};
\node[chip, minimum width=2.5cm, align=center, minimum height=0.45cm,
      draw=none, fill=none, text=white]
  (q3ph) at (5.8, 0.8)  {$\ldots$ $Q{=}12$ queries};

\node[draw=clrblue!90, fill=white, rounded corners=6pt,
      line width=1pt,
      fit=(q1ph)(q3ph), inner sep=7pt
     ] (qbox) {};

\node[chip, minimum width=2.5cm, text width=2.35cm, align=center, minimum height=0.55cm]
  (q1) at (5.8, 2.35) {``Austin lakefront\\ cabin with hot tub''};
\node[chip, minimum width=2.5cm, text width=2.35cm, align=center, minimum height=0.55cm]
  (q2) at (5.8, 1.55) {``Austin family cabin\\ near lake''};
\node[chip, minimum width=2.5cm, align=center, minimum height=0.45cm]
  (q3) at (5.8, 0.8)  {$\ldots$ $Q{=}12$ queries};

\node[stage]     (enc)   at (8.65,  1.55)  {E5\\ encoder};
\node[stage]     (faiss) at (10.8,  1.55)  {FAISS\\ 11.7M};
\node[stage]     (rrf)   at (12.95, 1.55)  {RRF +\\ geo re-rank};
\node[stage, fill=clrblue!25, minimum width=1.3cm]
                 (cllm)  at (14.95, 1.55)  {$C_{\text{LLM}}$};

\node[stage]     (ibknn) at (2.9, -1.8) {IBKNN\\ co-occurrence};
\node[stage, fill=clrblue!25, minimum width=1.3cm]
                 (cib)   at (14.95, -1.8) {$C_{\text{IB}}$};

\node[fusion]  (union) at (17.55, -0.1) {Union\\ fusion};
\node[stage, fill=clrorange!20]
               (pool)  at (19.85, -0.1) {Top-1000\\ pool};
\node[stageemph, fill=clrorange!55, draw=clrorange!90]
               (lgb)   at (22.1, -0.1)  {LightGBM\\ ranker};

\node[outbox] (topk) at (24.35, -0.1) {Top-$K$\\ results};

\begin{scope}[on background layer]
  \node[draw=clrblue!80, dashed, line width=0.6pt,
        fill=clrblue!4,
        rounded corners=4pt,
        fit=(llm)(qbox)(cllm),
        inner sep=8pt,
        label={[font=\footnotesize, text=black, yshift=-2pt]above:LLM channel (training-free)}
       ] (llmgroup) {};

  \node[draw=clrblue!80, dashed, line width=0.6pt,
        fill=clrblue!4,
        rounded corners=4pt,
        fit=(ibknn)(cib),
        inner sep=8pt,
        label={[font=\footnotesize, text=black, yshift=2pt]below:Behavioral channel (IBKNN)}
       ] (behgroup) {};

  \node[draw=clrorange!85, dashed, line width=0.6pt,
        fill=clrorange!5,
        rounded corners=4pt,
        fit=(union)(lgb),
        inner sep=8pt,
        label={[font=\footnotesize, text=black, yshift=-2pt]above:Fusion \& Re-ranking}
       ] (fusegroup) {};
\end{scope}

\draw[arr] (llm.east)   -- (qbox.west |- llm.east);
\draw[arr] (qbox.east |- enc.west) -- (enc.west);
\draw[arr] (enc.east)   -- (faiss.west);
\draw[arr] (faiss.east) -- (rrf.west);
\draw[arr] (rrf.east)   -- (cllm.west);

\draw[arr] (mdbox.east) -- (llmgroup.west |- llm.west) -- (llm.west);

\draw[arr] (clicks.east) -- (behgroup.west |- ibknn.west) -- (ibknn.west);
\draw[arr] (ibknn.east)  -- (cib.west);

\draw[arr] (cllm.east) -- (llmgroup.east |- cllm.east)
                      -- ($(llmgroup.east |- cllm.east)+(0.3,0)$)
                      |- (fusegroup.west |- union.west)
                      -- (union.west);
\draw[arr] (cib.east)  -- (behgroup.east |- cib.east)
                      -- ($(behgroup.east |- cib.east)+(0.3,0)$)
                      |- (fusegroup.west |- union.west)
                      -- (union.west);

\draw[arr] (union.east) -- (pool.west);
\draw[arr] (pool.east)  -- (lgb.west);

\draw[arr] (lgb.east) -- (fusegroup.east |- lgb.east) -- (topk.west);

\end{tikzpicture}%
}%
\Description{A pipeline diagram showing the training-free LLM candidate generation architecture. On the top row, a property metadata box (containing fields for Title, Description, City/State, Capacity, and Amenities) feeds into an LLM query generator, which produces Q=12 natural-language search queries. Two example queries are shown: ``Austin lakefront cabin with hot tub'' and ``Austin family cabin near lake''. The queries then pass through an E5 encoder, a FAISS index over 11.7M properties, and RRF plus geographic re-ranking to produce the LLM-channel candidate list C\_LLM. On the bottom row, a user clickstream feeds into an IBKNN co-occurrence module producing candidate list C\_IB. Both candidate lists feed into Union fusion, which produces a Top-1000 pool, passed to a LightGBM ranker, producing the final Top-K results. The LLM channel (top), Behavioral channel (bottom), and Fusion \& Re-ranking block (right) are each shown grouped in dashed labeled boxes.}
\caption{End-to-end candidate generation and re-ranking pipeline.
The LLM channel (top) converts static property metadata into $Q{=}12$
natural-language queries via an instruction-tuned LLM, encodes them
with E5, and retrieves candidates from the 11.7M-property FAISS index
with RRF + geographic re-ranking. The behavioral channel (bottom)
produces $C_{\text{IB}}$ from IBKNN on user clickstreams. Union fusion
(\S\ref{sec:method-union}) merges the two lists while preserving
IBKNN's ordering, and a downstream LightGBM ranker re-scores the
fused top-1000 pool.}
\label{fig:pipeline}
\end{figure*}
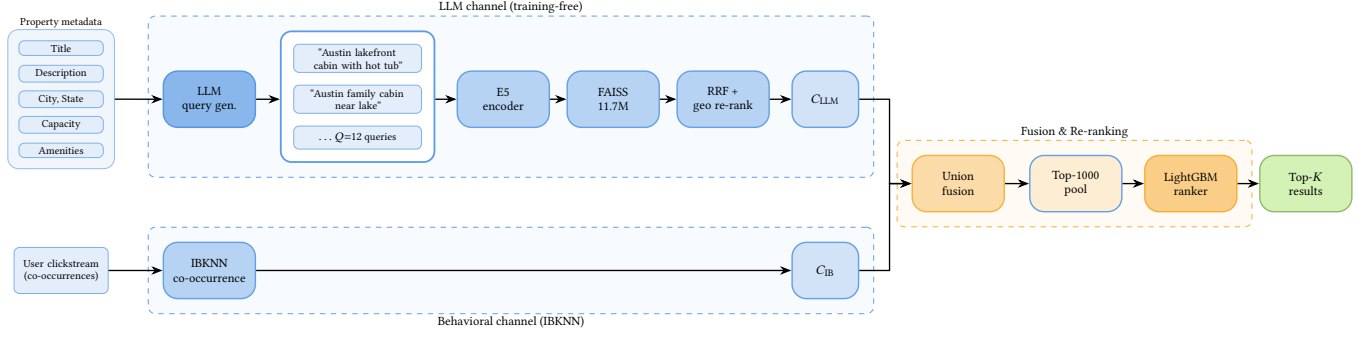
\subsection{Pipeline Overview}\label{sec:method-overview}
The LLM-based candidate generation pipeline operates entirely on
static property metadata and produces a ranked list of candidate
properties for any focal property in the catalog.\ It consists of four stages (Figure~\ref{fig:pipeline}), applied per focal property:
\begin{enumerate}[leftmargin=*]
\item \textbf{Query generation.} An off-the-shelf instruction-tuned
  LLM is prompted with the focal property's metadata and produces a
  small set of diverse natural-language search queries describing
  what the property offers from different angles
  (\S\ref{sec:method-querygen}). Instruction tuning is needed so the
  model reliably follows the prompt's length and format rules.
  \item \textbf{Query encoding.} Each generated query is encoded into
  a dense vector by a pre-trained text encoder. The catalog of
  11.7M properties has been pre-encoded once, offline, into the same
  embedding space.
  \item \textbf{Approximate nearest-neighbor retrieval.} Each query
  embedding is dispatched against a FAISS index over the property
  catalog, returning the top-$N$ nearest properties per query.
  \item \textbf{Aggregation and location re-ranking.} Per-query
  result lists are aggregated via Reciprocal Rank Fusion~\cite{cormack2009rrf} to produce
  a single ranked list of LLM-channel candidates, which is then
  re-ranked by a weighted combination of semantic similarity and
  geographic distance from the focal property.
\end{enumerate}

The output of the pipeline is the LLM channel's ranked candidate
list, $C_{\text{LLM}}$, for each focal property. This list is then
combined with the IBKNN channel's list $C_{\text{IB}}$ via the
Union fusion strategy formalized in \S\ref{sec:method-union}, and
the fused pool is re-scored by a downstream learning-to-rank model
(\S\ref{sec:method-ranker}). Because the pipeline depends only on
metadata, it is stateless, parallel across focal
properties, and applicable to any property in the catalog including
those with zero interaction history (\S\ref{sec:method-coldstart}).

\subsection{Query Generation from Property Metadata}\label{sec:method-querygen}
For each focal property, we prompt an off-the-shelf
instruction-following LLM with a compact serialization of the
property's structured metadata and ask it to produce $Q$ diverse
natural-language search queries. We perform no fine-tuning of our own; the pipeline relies only on the instruction-following
capability the model already has from its provider's training.
The LLM is treated as an interchangeable component of the
pipeline: any model capable of following the prompt's length and
format constraints can serve this role, and we evaluate several
choices in \S\ref{sec:results-llm-choice} to quantify how much
the specific model matters. 

\paragraph{Metadata serialization.} Each property is rendered as a
short, line-separated text block containing, in order: the listing
title, the description (HTML-stripped and truncated to 300
characters), the structure type and administrative location (city,
state, country), the market or destination area, the property's
capacity (bedrooms, bathrooms, sleeps), up to 15 amenities as a
comma-separated list, and the average rating with review count.
Fields that are missing or empty are omitted rather than padded,
so properties with sparse metadata still produce well-formed
prompts. The same serialization is used by the property-catalog
encoder (\S\ref{sec:method-overview}, Stage~2), which ensures that
the text seen by the LLM at query time is structurally aligned
with the text used to build the FAISS index.

\paragraph{Prompt design.} The prompt (Fig.~\ref{fig:prompt})
instructs the LLM to generate $Q$ search queries that a traveler
might plausibly use to find the focal property \emph{or nearby
alternatives}, with three explicit constraints: (i) every query
must include the city or region name, which anchors dense
retrieval geographically -- without this constraint, queries
drift toward catalog-wide descriptive matches and lose geographic
relevance; (ii) each query should highlight a different facet
of the property, spanning location-plus-type,
location-plus-amenity, and location-plus-experience variants;
and (iii) queries must be natural-language search phrases between
five and twelve words, without prices (e.g., ``under \$200/night''), numbering (e.g., ``1.'', ``2.'', \ldots), or meta-commentary (e.g., ``Here are some queries:''). At runtime, \texttt{\{num\_queries\}} is replaced
with $Q=12$, \texttt{\{focal\_text\}} with the serialized metadata,
and \texttt{\{history\_section\}} is empty in our item-to-item
setting.

\begin{figure}[h]
\begin{lstlisting}[
  basicstyle=\ttfamily\scriptsize,
  frame=single,
  breaklines=true,
  breakindent=0pt,
  xleftmargin=0pt,
  framesep=3pt,
]
Generate {num_queries} search queries that a traveler
would use to find this vacation rental or similar
alternatives.

Property:
{focal_text}

{history_section}Rules:
- EVERY query MUST include the city or region name
- Write natural search phrases (e.g., "Austin lakefront
  cabin with hot tub")
- Each query should highlight a different feature:
  location + property type, location + amenities,
  location + experience
- Keep each query between 5-12 words
- Do NOT include prices or dollar amounts
- Do NOT number the queries

Output ONLY {num_queries} queries, one per line.
\end{lstlisting}
\caption{Prompt template used to generate $Q=12$ diverse search
queries from focal property metadata.
\texttt{\{num\_queries\}}, \texttt{\{focal\_text\}}, and
\texttt{\{history\_section\}} are substituted at runtime.}
\label{fig:prompt}
\end{figure}


\paragraph{Why multiple queries rather than a single embedding.}
Generating $Q$ queries per property rather than embedding metadata
into a single dense vector is deliberate: a single embedding
collapses all facets into one point, while multi-query generation
forces the LLM to surface different facets as separate queries,
each retrieving a different slice of the catalog. The diversity
enters at the query level, and the downstream aggregation step
(\S\ref{sec:method-overview}, Stage~4) combines per-query results
via Reciprocal Rank Fusion.

\paragraph{Output parsing and validation.} Each candidate query
passes through a lightweight validation filter that strips
enumeration artifacts, rejects queries outside the 10--200
character range, removes queries containing price patterns or
AI-refusal phrases, and de-duplicates near-identical variants.
When the validated set contains fewer than $Q$ queries -- typically
because the LLM produced malformed output or the property's
metadata was extremely sparse -- the pipeline pads the remainder
with a deterministic fallback queries assembled directly from the property's metadata fields, each combining the location with a structural attribute (e.g., ``Austin cabin rental'', ``Austin cabin with hot tub'').


\subsection{Union Fusion with IBKNN}\label{sec:method-union}
Let $C_{\text{IB}} = (i_1, i_2, \ldots, i_{|C_{\text{IB}}|})$ denote
the ordered list of candidates returned by IBKNN for a given focal
property, and let
$C_{\text{LLM}} = (j_1, j_2, \ldots, j_{|C_{\text{LLM}}|})$ denote
the ordered list returned by the LLM channel after Stage 4 of the
pipeline. We define the Union fusion of these two lists at cutoff
$K$ procedurally as shown in Alg.~\ref{alg:union}.

\begin{algorithm}[t]
\DontPrintSemicolon
\KwIn{ordered lists $C_{\text{IB}}$, $C_{\text{LLM}}$;
  target output size $K$}
\KwOut{fused top-$K$ candidate list}
$\text{result} \leftarrow C_{\text{IB}}$
  \tcp*{copy IBKNN, preserve order}
$\text{seen} \leftarrow$ set of property ids in $C_{\text{IB}}$\;
\ForEach{$c \in C_{\text{LLM}}$}{
  \lIf{$|\text{result}| \geq K$}{\textbf{break}}
  \If{$c.\mathrm{id} \notin \text{seen}$}{
    append $c$ to $\text{result}$\;
    add $c.\mathrm{id}$ to $\text{seen}$\;
  }
}
\Return $\text{result}[:K]$\;
\caption{Union fusion of IBKNN and LLM-channel candidates.}
\label{alg:union}
\end{algorithm}

The Union output begins with IBKNN's list exactly as produced and
then extends it by walking through the LLM channel's list in
order, appending only items not already present in the IBKNN
list, stopping once the total length reaches $K$. Three
properties follow directly from the algorithm and are central to
what comes next: (i) IBKNN's ordering is preserved exactly --
no item is ever reordered, replaced, or removed; (ii) when an LLM
candidate duplicates an IBKNN candidate, the IBKNN copy is the
one retained; and (iii) the LLM channel only fills slots that
IBKNN did not reach.

\begin{theorem*}[Recall preservation]\label{thm:recall-preservation}
For every focal property $f$ and every $K \geq 1$,
\[
\mathrm{Recall}@K\bigl(\mathrm{Union}_K(C_{\text{IB}}^{(f)},
  C_{\text{LLM}}^{(f)})\bigr) \;\geq\;
\mathrm{Recall}@K\bigl(C_{\text{IB}}^{(f)}\bigr),
\]
with equality when $|C_{\text{IB}}^{(f)}| \geq K$.
\end{theorem*}

\begin{proof}
Let $m = \min(K, |C_{\text{IB}}^{(f)}|)$. By construction, the
fused list begins with the first $m$ items of $C_{\text{IB}}^{(f)}$
in their original order. Therefore, every relevant item that IBKNN
places in its top-$K$ also appears in the fused top-$K$, which
gives the inequality. When $|C_{\text{IB}}^{(f)}| \geq K$, the
algorithm's early stop on line~4 fires before any LLM candidate is
appended, so the fused list equals the first $K$ items of
$C_{\text{IB}}^{(f)}$ exactly, and equality holds.
\end{proof}




A stronger \emph{position preservation} property follows from the
same construction: since Alg.~\ref{alg:union} only appends to
$\text{result}$ and never reorders, any IBKNN item at rank $r \leq K$
occupies position $r$ in the Union output as well. This makes the theorem tight -- Union matches rather than strictly improves on IBKNN's $\mathrm{Recall}@K$ only when
IBKNN already fills all $K$ slots; otherwise, every relevant
LLM-contributed item in a previously-empty slot improves recall.


One scope note is worth making explicit: the guarantee is on the
fused candidate set, not on the final ranked output. The
downstream learning-to-rank model
(\S\ref{sec:method-ranker}) is free to reorder anything in the
fused pool, and its precision behavior is evaluated separately
from the candidate-set guarantees stated here.

\subsection{LightGBM Re-ranker}\label{sec:method-ranker}
A downstream LightGBM ranker re-scores the Union-fused top-1000
pool to produce the final ranked output. The ranker uses features
spanning rank and score priors from each upstream channel,
geographic distance, per-channel similarity signals (IBKNN
co-occurrence strength and cosine similarity in the dense
text-embedding space described in \S\ref{sec:method-overview},
Stage~2), property-level demand profiles, user-segment preferences, and focal-aware
pairwise interactions.\ The goal of the ranker is
to add incremental value on top of the fused candidate pool by
exploiting features the upstream channels do not directly model.
Detailed feature design and training choices are not the focus of
this paper.

Each (focal, candidate) training pair is assigned a graded
relevance label for the LambdaRank objective: 4 if the two
properties were booked in the same session, 1 if co-viewed, and 0
otherwise. Negatives are sampled uniformly at 5:1 against
positives, capped at 5M pairs. Property-level features (review counts, ADR,
booking volumes) use rolling 7--90-day windows snapshotted to the
training cutoff to avoid future-information leakage. The training
window is disjoint from evaluation. We train a LightGBM LambdaRank
model grouped by focal property with NDCG@\{5, 10, 50\} as the
training metric. \S\ref{sec:results-ranker} covers the evaluation results.

\subsection{Cold-Start Coverage by Construction}\label{sec:method-coldstart}
A structural property of the LLM channel is that its eligibility
criterion for serving a focal property is fundamentally different
from IBKNN's. IBKNN requires that the focal property has appeared
in user clickstreams alongside other properties; without
co-occurrence data, IBKNN returns an empty list and the property
receives no candidates. The LLM channel, by contrast, requires
only that the focal property has parseable metadata: a location,
a structure type, and enough descriptive text for the LLM to
generate queries from. Whether any user has ever interacted with
the property is irrelevant to whether candidates can be generated
for it.

This shifts cold-start coverage from a remediation problem to a
structural one. Any property the catalog admits is eligible for
the LLM channel by construction, including brand-new listings
created hours before the channel runs, properties in emerging
markets that have not yet accumulated traffic, and seasonal
inventory dormant during the IBKNN aggregation window. The LLM
channel does not need to ``learn'' how to handle these properties;
it handles them by virtue of how the eligibility is defined. We
quantify the resulting coverage gain in \S\ref{sec:results}.

\subsection{Implementation Details}\label{sec:method-impl}
We report the specific models, infrastructure, and parameters used
in the experiments reported in this work.

\paragraph{Query generator.} The full-scale pipeline uses
LLaMA-3.2-3B-Instruct \cite{llama32} as the query-generation
LLM, served via vLLM behind an OpenAI-compatible HTTP endpoint
on a single-node GPU cluster. Generation runs with temperature $T=0.4$, $Q=12$ queries per focal property, and a per-request \texttt{max\_tokens} cap of 300. The prompt (Fig.~\ref{fig:prompt}) and the parser and validation filter described in \S\ref{sec:method-querygen} are applied uniformly to the LLM's raw output.\ For the LLM-choice
ablation in \S\ref{sec:results-llm-choice}, we additionally compare
against GPT-4o and GPT-5.2 accessed through an Azure OpenAI proxy,
using the identical prompt template, sampling parameters, parser,
and validation filter; only the underlying model identity changes
across the comparison. The self-hosted model generates queries for
a 10K focal-property batch in roughly five minutes on a single GPU
node and scales to the full 1.6M focal set on the same hardware.
The frontier API endpoints, capped at 50 requests per minute, could
not be run at full catalog scale within a practical time budget,
which is why the LLM-choice ablation
(\S\ref{sec:results-llm-choice}) is restricted to a 10K subset.

\paragraph{Text encoder and FAISS index.} Property metadata and
generated queries are encoded with E5-large-v2~\cite{wang2022e5},
a 1024-dimensional sentence encoder.\ E5 is trained for asymmetric
retrieval and expects each input to be tagged by role: we prepend
the \texttt{passage:} prefix to catalog entries (the documents
being searched) and the \texttt{query:} prefix to LLM-generated
queries, as the model requires.\ The full 11.7M-property catalog is
encoded once offline with batch size 64 and L2-normalized, then
indexed with FAISS using an exact inner-product index
(\texttt{IndexFlatIP}). We use a flat index rather than an
approximate one because the catalog is small enough to fit in
GPU memory and the exact index removes approximation error as a
confounding factor in the cross-method recall comparisons. Query
retrieval fetches the top-50 candidates per query; per-query
result lists are aggregated via Reciprocal Rank Fusion with
$k=60$ (\S\ref{sec:method-overview}, Stage~4) and then geographically re-ranked with a $0.7{:}0.3$ weighting of semantic similarity and normalized haversine distance from the focal property.


\section{Experimental Setup}\label{sec:setup}
\paragraph{Dataset and catalog.} We evaluate on a Vrbo
property catalog of approximately 11.7M listings. The evaluation set consists of
1.6M focal properties drawn from this catalog, spanning the full
range of interaction volumes from cold-start listings with zero
interactions to heavily-trafficked properties with hundreds of
clickstream events per month. The focal set was selected by
filtering to properties with parseable metadata (a required
prerequisite for the LLM channel) and restricting to Vrbo brand
inventory; no other filters were applied.

\paragraph{Evaluation protocol.} Ground-truth relevance for each
focal property is derived from the held-out clickstream window
\emph{2025-09-09 to 2025-10-08}, using a co-occurrence definition:
for a focal property $f$, any property $p$ that appears in the
same user session as $f$ during the held-out window is treated
as a relevant candidate for $f$. The training window used by the
behavioral baseline (IBKNN) is \emph{2025-06-08 to 2025-09-08},
strictly disjoint from the evaluation window. We report
$\mathrm{Recall}@K$ as the primary metric, averaged uniformly
across all focal properties in the evaluation set, at
$K \in \{5, 10, 50, 100, 200, 300, 500, 1000\}$, with $K = 1000$
as the primary operating point. As a secondary metric, we report
the mean haversine distance from the focal property to its
returned candidates, averaged at $K = 5$ and $K = 10$, to capture
geographic precision -- a property that matters in marketplace
lodging where travelers typically care about location as much as
attributes.

\paragraph{Baselines.} Our primary baseline is \textbf{IBKNN}
(behavioral co-occurrence with amenity boost and 30~km
geo-prefilter), a core component of Vrbo's CG stack.
The proposed method is \textbf{Union fusion} of the LLM channel
with IBKNN (Alg.~\ref{alg:union}). For the LLM-choice
ablation in \S\ref{sec:results-llm-choice}, we additionally
compare against two non-Union fusion strategies: \textbf{LLM-only}
(no fusion) and \textbf{Reciprocal Rank Fusion} (RRF), a
rank-blending strategy that combines both channels into every
output position.


\paragraph{Evaluation variants.} For the per-segment comparisons
in \S\ref{sec:results-segment}, we use a \emph{shared-population}
protocol that restricts evaluation to the 1.5M properties
where both IBKNN and Union have candidates, providing a fair comparison that isolates retrieval quality from
coverage differences. Cold-start properties where only the LLM
channel has candidates are reported separately in
\S\ref{sec:results-coldstart}.

\paragraph{LLM-ablation subset.} The LLM-choice ablation in
\S\ref{sec:results-llm-choice} is run on a 10K focal-property
subset drawn uniformly at random (with fixed seed) from the
intersection of the 1.6M focal set and properties appearing in
the held-out clickstream window. This smaller scale is forced by
the rate-limit constraints of frontier API-based LLMs, described
in \S\ref{sec:method-impl}. LLaMA-3.2-3B's numbers on this subset
are extracted by filtering the full 1.6M run to the 10K focal
IDs, so LLaMA at 10K is the same model output as LLaMA at 1.6M,
restricted to an overlapping focal set. We confirm the subset is
representative of the full 1.6M by verifying that LLaMA's recall
on the 10K subset matches its aggregate recall on the full focal
set within 1.7pp, and report GPT-4o and GPT-5.2
comparisons only on the 10K subset.

\section{Results}\label{sec:results}
We evaluate the pipeline along three dimensions: aggregate recall
on the full focal set (\S\ref{sec:results-overall}), per-segment
recall stratified by focal-property interaction volume
(\S\ref{sec:results-segment}), and coverage of cold-start
properties that IBKNN cannot serve (\S\ref{sec:results-coldstart}).
We then examine how robust the pipeline is to the choice of
query-generator LLM (\S\ref{sec:results-llm-choice}) and the
incremental lift from the downstream LightGBM re-ranker
(\S\ref{sec:results-ranker}). Throughout, IBKNN is the behavioral
baseline and Union is the proposed fusion strategy
(Algorithm~\ref{alg:union}); per-segment and ranker comparisons
use the shared-property protocol from \S\ref{sec:setup}.
\subsection{Overall Performance}\label{sec:results-overall}
Table~\ref{tab:overall} reports $\mathrm{Recall}@K$ for the three
pipeline components -- IBKNN alone, the LLM channel alone, and
Union fusion -- across the full 1.6M focal-property evaluation
set. Union matches or beats IBKNN at every $K$, consistent with the recall-preservation guarantee.\ The standalone LLM channel underperforms IBKNN by a wide margin, as expected for a channel with no access to behavioral signal, but the fused Union output
recovers the full behavioral baseline and adds incremental recall
from the LLM channel's non-overlapping candidates.
\begin{table}[t]
\caption{Overall performance on the full 1.6M focal set.\ Union $\geq$ IBKNN at every $K$ by construction.\ The LLM channel alone
underperforms IBKNN but contributes non-overlapping candidates
that lift Union above IBKNN at every operating point.
\emph{R} = Recall, \emph{H} = mean haversine distance to
retrieved candidates. $^*$The LLM channel produces ${\sim}500$
unique candidates per focal ($Q{=}12$ queries $\times$ top-50
retrievals, after deduplication), so its R@1000 equals R@500.}
\label{tab:overall}
\small
\begin{tabular}{lcccc}
\toprule
 & IBKNN & LLM-only$^*$ & Union & $\Delta$ \\
\midrule
R@100   & 32.02\% & 6.99\%  & 32.26\% & +0.24pp \\
R@300   & 45.13\% & 12.02\% & 45.84\% & +0.71pp \\
R@1000  & 55.08\% & 13.50\% & 56.48\% & +1.40pp \\
\midrule
H@5 (km)  & 5.3 & 7.68 & 5.6 & +0.3 \\
H@10 (km) & 5.7 & 8.99 & 6.1 & +0.4 \\
\bottomrule
\end{tabular}
\end{table}

Union's recall lift over IBKNN grows monotonically with $K$: from
$+0.24$ percentage points at $\mathrm{R}@100$ to $+1.40$
percentage points at $\mathrm{R}@1000$. This pattern is a direct
consequence of Union's append-only design
(Alg.~\ref{alg:union}): at shallow $K$, IBKNN's candidates
fill most or all of the available slots and the LLM channel has
little room to contribute; at larger $K$, IBKNN's list is
exhausted and LLM-contributed candidates begin occupying tail
positions, each one an additional chance to recover a relevant
item. The two channels' candidate pools are highly complementary,
with only 4\% mean overlap per focal property, confirming that
the LLM channel surfaces genuinely different candidates rather
than duplicating IBKNN's behavioral neighbors.

The standalone LLM channel's recall -- roughly one-fifth of
IBKNN's -- underscores its role as a complement, not a
replacement. Its candidate pool is bounded at ${\sim}500$ unique
candidates per focal ($Q{=}12 \times$ top-50, after dedup), so
recall plateaus beyond $K=500$. Union does not share this ceiling
because IBKNN contributes additional candidates in deeper
positions. Union's haversine is modestly higher than IBKNN's
($+0.3$~km at H@5), reflecting the LLM channel's broader
geographic distribution.

\subsection{Per-Segment Analysis}\label{sec:results-segment}
We stratify focal properties into four interaction-volume buckets
by the number of clickstream events each property received during
the training window: \emph{very~low} (1--5 interactions),
\emph{low} (6--20), \emph{medium} (21--100), and \emph{high}
(101+). Cold-start properties (zero interactions) are handled as
a separate coverage category rather than a fifth bucket
(\S\ref{sec:results-coldstart}), because IBKNN returns an empty
candidate list for these properties by construction and
per-property recall is not the right lens for measuring the
improvement. Bucket populations in our focal set are: very\_low
22\%, low 26\%, medium 34\%, high 18\%, with an additional 3.2\%
cold-start properties.

The aggregate numbers in \S\ref{sec:results-overall} obscure a
structural property of the improvement: Union's lift is not
distributed evenly across focal properties but concentrates on
the long tail, precisely where IBKNN is weakest.
Table~\ref{tab:per-segment} reports $\mathrm{Recall}@K$ broken
out by interaction-volume bucket on the shared-property set
(approximately 1.5M properties where both IBKNN and Union have
candidates), providing an apples-to-apples comparison that isolates retrieval
quality from coverage differences.
\begin{table}[t]
\caption{Per-segment $\mathrm{Recall}@K$ on the shared-property
set (approximately 1.5M focal properties), stratified by interaction
volume. Union's lift is inversely proportional to interaction
count: largest on the sparsest bucket and negligible on the
densest. \emph{R} = Recall. 
}
\label{tab:per-segment}
\small
\begin{tabular}{llccc}
\toprule
Bucket & Method & R@100 & R@300 & R@1000 \\
\midrule
\multirow{3}{*}{\shortstack[l]{very\_low\\(1--5)}}
  & IBKNN & 31.48\% & 40.66\% & 44.55\% \\
  & Union & 33.10\% & 45.00\% & 50.80\% \\
  & $\Delta$ & +1.62pp & +4.34pp & +6.25pp \\
\midrule
\multirow{3}{*}{\shortstack[l]{low\\(6--20)}}
  & IBKNN & 32.94\% & 44.39\% & 50.28\% \\
  & Union & 33.49\% & 46.40\% & 53.89\% \\
  & $\Delta$ & +0.55pp & +2.01pp & +3.61pp \\
\midrule
\multirow{3}{*}{\shortstack[l]{medium\\(21--100)}}
  & IBKNN & 35.06\% & 48.52\% & 57.51\% \\
  & Union & 35.19\% & 49.15\% & 59.13\% \\
  & $\Delta$ & +0.13pp & +0.63pp & +1.62pp \\
\midrule
\multirow{3}{*}{\shortstack[l]{high\\(101+)}}
  & IBKNN & 32.69\% & 46.97\% & 58.93\% \\
  & Union & 32.72\% & 47.11\% & 59.44\% \\
  & $\Delta$ & +0.03pp & +0.14pp & +0.51pp \\
\bottomrule
\end{tabular}
\end{table}

The \emph{very\_low} bucket is the headline finding: Union
delivers a $+4.34$ percentage point lift in $\mathrm{R}@300$ over
IBKNN on properties with 1--5 interactions, growing to $+6.25$pp
at $\mathrm{R}@1000$. This is the segment where a marketplace
operator most needs help: these properties generate too little
behavioral signal for IBKNN to build reliable neighborhoods, yet
they are established enough that they should not be treated as
new listings. The \emph{low} bucket (6--20 interactions) shows a
similar pattern at reduced magnitude ($+2.01$pp at
$\mathrm{R}@300$), confirming that the LLM channel's value
extends beyond the extreme tail.

At the other end of the spectrum, the \emph{high} bucket (101+
interactions) shows Union's lift shrinking to $+0.14$pp at
$\mathrm{R}@300$ and $+0.51$pp at $\mathrm{R}@1000$. This is
exactly the behavior predicted by Union's design: on
well-interacted properties, IBKNN already fills $K$ slots with
strong behavioral candidates, leaving the LLM channel little room
to contribute. Early recall ($\mathrm{R}@5$, $\mathrm{R}@10$; not
shown) is identical between IBKNN and Union across all four
buckets on shared properties, empirically confirming the position-preservation property (\S\ref{sec:method-union}).

The monotonic relationship between interaction volume and Union's
lift is not a lucky outcome of tuning. It is a direct structural
consequence of Union fusion's append-only design: the LLM channel
fills only the slots that IBKNN leaves empty, and IBKNN leaves
more slots empty on sparse properties than on dense ones. The
theorems in \S\ref{sec:method-union} predict exactly this
asymmetric lift profile, and the per-segment numbers confirm it
empirically. Union additionally serves cold-start
properties that IBKNN cannot cover at all; their recall is
reported separately in \S\ref{sec:results-coldstart}.

\subsection{Cold-Start Coverage}\label{sec:results-coldstart}
The per-segment recall analysis in \S\ref{sec:results-segment}
deliberately excludes properties that IBKNN cannot serve at all.
This subsection quantifies that excluded population: its size,
its distribution across interaction-volume buckets, and why it
matters for the marketplace-level interpretation of our results.
On the full 1.6M focal set, IBKNN returns an empty candidate list
for roughly 3.2\% of the focal population -- properties with zero
behavioral co-occurrences during the training window, typically
new listings, dormant seasonal inventory, or properties in
emerging markets. Union fusion serves candidates for every one of
these properties, because the LLM channel's eligibility criterion
depends only on metadata parseability, not on interaction history
(\S\ref{sec:method-coldstart}).
\begin{table}[h]
\caption{Cold-start properties served by the LLM channel that
IBKNN cannot cover, broken out by interaction-volume bucket.}
\label{tab:coldstart-dist}
\small
\begin{tabular}{lcc}
\toprule
Bucket & Cold-start properties & \% of bucket \\
\midrule
very\_low (1--5)   & ${\sim}19$K & 5.3\% \\
low (6--20)        & ${\sim}15$K & 3.6\% \\
medium (21--100)   & ${\sim}13$K & 2.3\% \\
high (101+)        & ${\sim}4$K  & 1.3\% \\
\midrule
Total              & ${\sim}51$K & 3.2\% \\
\bottomrule
\end{tabular}
\vspace{-1.0em}
\end{table}
These cold-start properties are not concentrated in the lowest
interaction bucket. As Table~\ref{tab:coldstart-dist} shows,
cold-start properties appear across all four segments, including
roughly 1.3\% of the high bucket. This occurs because the buckets
are defined by \emph{test-period} interaction counts (Sep~9 --
Oct~8), while IBKNN's candidate generation uses
\emph{training-period} co-occurrences (Jun~8 -- Sep~8). A
property can receive interactions during the test window -- and
thus land in a non-zero bucket -- while having zero
co-occurrences during the training window that IBKNN relies on.
The LLM channel covers these properties preemptively, regardless
of when their behavioral signal materializes.

For these cold-start properties, Union's candidates come entirely
from the LLM channel. Note that the aggregate gains in
\S\ref{sec:results-overall} and per-segment gains in
\S\ref{sec:results-segment} are computed on the shared-property
set and \emph{exclude} these cold-start properties -- the true
marketplace-level improvement is larger than those tables suggest.


\subsection{LLM Choice Robustness Under Union Fusion}\label{sec:results-llm-choice}
The LLM channel is an interchangeable pipeline component
(\S\ref{sec:method-querygen}), so a natural question for a
deployment-conscious reader is how much the choice of
query-generator LLM matters to downstream recall. At marketplace
catalog scale, this is not a purely academic question: the rate-limit and cost constraints
of frontier API-based models make a full catalog daily refresh impractical, whereas a self-hosted open-weights model runs the same workload on a single GPU node (\S\ref{sec:method-impl}). We
compare three instruction-tuned LLMs as drop-in query generators,
holding the prompt template, sampling parameters, parser,
validation filter, and evaluation protocol constant across all
three. The only variable is the model identity.

Table~\ref{tab:llm-ablation} reports the comparison across three
fusion strategies: the LLM channel alone (no fusion with IBKNN),
Reciprocal Rank Fusion, and Union fusion. The pattern is striking.
\begin{table}[t]
\caption{LLM choice comparison on the 10K focal subset across
three fusion strategies. The standalone LLM-quality gap collapses
under Union fusion, while geographic precision shows a
fusion-strategy trade-off: the LLM channel alone is geographically
tightest but has the lowest recall; Union maximizes recall at the
cost of a modest haversine increase. All rows computed on the
filtered 10K test denominator for direct comparability across
fusion blocks. \emph{R} = Recall, \emph{H} = mean haversine
distance to retrieved candidates.}
\label{tab:llm-ablation}
\small
\begin{tabular}{lcccc}
\toprule
Fusion & Metric & LLaMA-3B & GPT-4o & GPT-5.2 \\
\midrule
\multirow{5}{*}{Standalone} & R@100     & 7.28\%  & 9.74\%  & 9.42\%  \\
                            & R@300     & 12.42\% & 16.00\% & 16.52\% \\
                            & R@1000    & 13.95\% & 17.65\% & 18.82\% \\
                            & H@5 (km)  & 3.65    & 2.19    & 2.27    \\
                            & H@10 (km) & 4.83    & 2.66    & 4.22    \\
\midrule
\multirow{5}{*}{RRF}        & R@100     & 22.42\% & 23.47\% & 23.34\% \\
                            & R@300     & 34.83\% & 36.05\% & 36.04\% \\
                            & R@1000    & 40.68\% & 41.90\% & 42.06\% \\
                            & H@5 (km)  & 3.80    & 3.17    & 3.12    \\
                            & H@10 (km) & 4.22    & 3.52    & 3.42    \\
\midrule
\multirow{5}{*}{Union}      & R@100     & 33.58\% & 33.66\% & 33.64\% \\
                            & R@300     & 47.50\% & 47.72\% & 47.74\% \\
                            & R@1000    & 58.57\% & 58.95\% & 59.04\% \\
                            & H@5 (km)  & 5.35    & 5.24    & 5.24    \\
                            & H@10 (km) & 5.79    & 5.68    & 5.69    \\
\bottomrule
\end{tabular}
\end{table}

Used standalone, the LLM channel shows a wide quality gap:
GPT-4o outperforms LLaMA-3.2-3B by 27--46\% in $\mathrm{Recall}@K$
and GPT-5.2 by 25--35\%; haversine distance improves by 40--45\%
for GPT-4o. Under RRF the gap partially closes but frontier models
retain a 3--10\% edge. Under Union it collapses: all three models
land within 0.4pp at every $K$ from 100 to 1000 (a 38$\times$
reduction in the quality gap versus standalone), with haversine
tightening to within $\sim$2\%. The haversine rows also show a
fusion-strategy trade-off: the LLM channel alone is geographically
tightest (H@5: 2.2--3.7~km), RRF intermediate, Union widest
(5.2--5.4~km) because IBKNN dominates the head positions. This
collapse is exactly what the position-preservation property
(\S\ref{sec:method-union}) predicts: IBKNN occupies the top
positions where most recall mass lives, leaving the LLM's quality
ceiling little room to express itself in the fused output.


The operational implication is that a marketplace operator can
deploy the cheapest, fastest query generator available without
materially compromising recall, provided Union fusion is the
channel-combination strategy -- at 1.6M-property scale, daily
refresh with a frontier API-based LLM is infeasible under current
rate limits (\S\ref{sec:method-impl}), while a self-hosted 3B
model completes in hours on a single GPU node.

\subsection{LightGBM Re-ranker}\label{sec:results-ranker}
The Union-fused top-1000 pool is re-scored by the downstream
LightGBM ranker described in \S\ref{sec:method-ranker}. The
role of this stage is to add complementary incremental value on top of the
fused pool by exploiting features (geographic distance, property
similarity, demand profiles) that the upstream channels do not
directly model. Table~\ref{tab:ranker} reports the
best-performing ranker variant's lift over the pre-ranker Union
fused pool.
\begin{table}[t]
\caption{Downstream LightGBM ranker performance over the
Union-fused top-1000 pool on the 1.6M focal set. The ranker
delivers a small but consistent lift at $K \geq 10$, with peak
absolute lift in the mid-$K$ range, and converges to Union's
R@1000 because re-ranking cannot expand the candidate pool.}
\label{tab:ranker}
\small
\begin{tabular}{lcccc}
\toprule
 & R@10 & R@50 & R@100 & R@300 \\
\midrule
Union (pre-ranker) & 9.91\%  & 23.97\% & 32.26\% & 45.84\% \\
Union + LightGBM   & 10.08\% & 25.26\% & 33.95\% & 47.51\% \\
$\Delta$           & +0.17pp & +1.29pp & +1.69pp & +1.67pp \\
\bottomrule
\end{tabular}
\end{table}

The ranker improves Union's recall at every $K$ from 10 through
500, with the lift staying above $+1.2$ percentage points across
the entire $K = 50$--$500$ range and peaking near the middle. At
$K=1000$ the two systems are identical because re-ranking
redistributes candidates within the pool but cannot add new ones
-- the candidate generator's recall is the ceiling, and the
ranker's role is to push relevant candidates into earlier
positions within that ceiling. The mid-$K$ range where the lift
concentrates corresponds to the operating points most relevant
for downstream surface ranking, where the top hundreds of
candidates are re-scored by additional business-logic models. We
also note that the ranker introduces a small geographic-precision
cost relative to Union (mean haversine at $K=5$ rises from
5.60~km to 7.27~km), reflecting that property-level demand and
similarity features pull the model slightly away from raw
geographic proximity. The candidate-set guarantees from
\S\ref{sec:method-union} are unaffected by this stage: the ranker
is free to reorder anything in the fused pool, but the pool
itself preserves IBKNN's coverage and ordering by construction.

\section{Discussion}\label{sec:discussion}
\paragraph{Deployment implications for two-sided marketplaces.}
The results in \S\ref{sec:results} support a specific deployment
pattern for LLM-based candidate generation in marketplace
settings: use a mature behavioral channel (IBKNN) as the primary
signal, complement it with an LLM channel via Union fusion, and
re-rank the fused pool with a learned model. This stack delivers
three properties that marketplace operators need simultaneously:
(i) no regression on well-served properties (recall-preservation guarantee), 
(ii) incremental recall on the long tail where behavioral signal is sparse ($+4.34$pp at $\mathrm{R}@300$ on the very\_low segment), and (iii) complete
coverage of cold-start properties that behavioral methods cannot
reach. The Union fusion guarantee is especially important in marketplace contexts where per-property regression translates directly into host-side exposure loss -- a
concern that average-recall improvements do not address.


\paragraph{Limitations.} Several limitations merit
acknowledgment. First, our evaluation uses offline recall against
held-out clickstream co-occurrences; online metrics such as
click-through rate and booking conversion may tell a different
story, particularly for cold-start properties where the
ground-truth signal is inherently sparse. Second, the LLM
channel's geographic precision is lower than IBKNN's
(\S\ref{sec:results-llm-choice}), reflecting that
metadata-derived queries retrieve from a broader geographic
distribution than IBKNN's 30~km neighborhood; surfaces where
geographic adjacency is paramount may require tighter location
filtering. Third, the pipeline currently generates queries once
and serves them statically; it does not incorporate user-level
personalization or session context, which limits its ability to
adapt recommendations to individual traveler intent. Fourth, the
LLM channel's candidate pool is bounded by its retrieval
configuration ($Q = 12$ queries $\times$ top-50 per query
$\approx 500$ unique candidates); expanding retrieval depth or
adding sparse retrieval channels could raise this ceiling.

\paragraph{Future work.} Three directions follow naturally. First,
online A/B testing of the Union-fused pipeline against the IBKNN
baseline would validate whether the offline recall gains
translate to user-facing engagement improvements. Second,
iterative query refinement -- feeding retrieved candidates back
to the LLM to generate improved queries -- could lift retrieval
quality for properties where first-pass queries miss relevant
facets.\ Third, prompt-variant optimization (several facet-diversity
variants remain unevaluated) could further improve candidate
quality without requiring model fine-tuning.

\section{Conclusion}\label{sec:conclusion}
We presented a training-free, LLM-based candidate generation
pipeline for item-to-item recommendation in vacation rental
marketplaces, evaluated at scale on 1.6M focal properties over an
11.7M-property catalog. The pipeline uses an off-the-shelf LLM to
synthesize diverse search queries from property metadata, retrieves
candidates via dense search, and merges them with a behavioral
baseline through Union fusion, a channel-combination strategy
that preserves the behavioral channel's ordering by construction.

Four findings emerge from the evaluation. First, Union fusion
improves recall over IBKNN at every $K$ on shared properties,
with the lift growing from $+0.24$pp at $\mathrm{R}@100$ to
$+1.40$pp at $\mathrm{R}@1000$, driven by only 4\% candidate
overlap between the two channels. Second, the improvement
concentrates on the long tail: $+4.34$pp at $\mathrm{R}@300$ on
properties with 1--5 interactions, diminishing to $+0.14$pp on
the densest segment, exactly the asymmetric profile predicted
by Union's append-only design. Third, the LLM channel extends
coverage to the cold-start properties that IBKNN cannot serve,
a structural consequence of metadata-only eligibility rather than
a trained behavior. Fourth, the choice of query generator LLM matters far less under Union than it appears in isolation. Used standalone, a 3B open-weights model trails by 27--46\% in recall. Under Union fusion, the gap collapses to under 1\%. This makes self-hosted small-model deployment a defensible choice at marketplace catalog scale.

A downstream LightGBM re-ranker adds incremental lift over the
Union-fused pool ($+1.69$pp at $\mathrm{R}@100$), confirming
that the fused candidate set is amenable to learned re-scoring. Our proposed solution addresses the supply-side exposure imbalance
that motivates this work: long-tail and cold-start properties receive candidates without regressing well-served, high interaction properties.






\bibliographystyle{ACM-Reference-Format}
\bibliography{sample-base}

\end{document}